\documentclass{article} % For LaTeX2e

\newif\ifsup\suptrue      % Supplementary version or no

\usepackage{nips14submit_e,times}
\usepackage{hyperref}
\usepackage{url}
\hypersetup{
    bookmarks=true,         % show bookmarks bar?
    unicode=false,          % non-Latin characters in AcrobatÕs bookmarks
    pdftoolbar=true,        % show AcrobatÕs toolbar?
    pdfmenubar=true,        % show AcrobatÕs menu?
    pdffitwindow=false,     % window fit to page when opened
    pdfstartview={FitH},    % fits the width of the page to the window
    pdftitle={My title},    % title
    pdfauthor={Author},     % author
    pdfsubject={Subject},   % subject of the document
    pdfcreator={Creator},   % creator of the document
    pdfproducer={Producer}, % producer of the document
    pdfkeywords={keyword1} {key2} {key3}, % list of keywords
    pdfnewwindow=true,      % links in new window
    colorlinks=true,       % false: boxed links; true: colored links
    linkcolor=red,          % color of internal links (change box color with linkbordercolor)
    citecolor=blue,        % color of links to bibliography
    filecolor=magenta,      % color of file links
    urlcolor=cyan           % color of external links
}

\usepackage[bf,scriptsize,skip=2pt]{caption}
\usepackage{latexsym}
\usepackage{amsmath}
\usepackage{amssymb}
\usepackage{mathtools}
\usepackage{enumerate}
\usepackage{color}
\usepackage{stmaryrd}
\usepackage{pgfplots}
\usepackage{wrapfig}
\usepackage{algorithmicx}
\usepackage[noend]{algpseudocode}
\usepackage[ruled]{algorithm}
\usepackage{dsfont}
\newcommand{\citep}[1]{\cite{#1}}

\usepackage{tikz}
\usetikzlibrary{patterns}
\usetikzlibrary{decorations.shapes}
\usetikzlibrary{decorations.markings}
\usetikzlibrary{shapes}
\usepackage[disable]{todonotes}
\newcommand{\todot}[2][]{\todo[inline,color=blue!20!white,#1]{#2}}

\newcommand{\defined}{\vcentcolon =}
\newcommand{\rdefined}{=\vcentcolon}
\newcommand{\E}{\mathbb E}

\newcommand{\R}{\mathbb R}

\newcommand{\N}{\mathbb N}
\newcommand{\sr}[1]{\stackrel{#1}}
\newcommand{\set}[1]{\left\{#1\right\}}
\newcommand{\ind}[1]{\mathds{1}\!\set{#1}}

\newcommand{\argmax}{\operatornamewithlimits{arg\,max}}

\newcommand{\floor}[1]{\left \lfloor {#1} \right\rfloor}
\newcommand{\ceil}[1]{\left \lceil {#1} \right\rceil}
\newcommand{\KL}{\operatorname{KL}}

\newcommand{\eqn}[1]{\begin{align}#1\end{align}}
\newcommand{\eq}[1]{\begin{align*}#1\end{align*}}

\newcommand{\etal}{et.\ al.\ }

\def\subsubsect#1{\vspace{1ex plus 0.5ex minus 0.5ex}\noindent{\bf\boldmath{#1.}}}

\usepackage[amsmath,amsthm,thmmarks]{ntheorem}

\theoremstyle{plain}
\newtheorem{theorem}{Theorem}

\newtheorem{lemma}[theorem]{Lemma}

\theoremstyle{definition}

\newtheorem{remark}[theorem]{Remark}
\theoremstyle{remark}

\let\epsilon\varepsilon

\renewcommand{\P}[1]{\mathbb{P}\left\{#1\right\}}

\newcommand{\best}{{i^*}}
\newcommand{\gap}[1]{\Delta_{#1}}
\newcommand{\mingap}{\Delta_{\min}}
\newcommand{\maxgap}{\Delta_{\max}}
\newcommand{\true}{\theta^*}

\newcommand{\drawsimplegraph}{
\draw (0,0) edge[->] (4.2,0);
\draw (0,0) edge[->] (0,4.2);

\node[yshift=-0.3cm] at (0,0) {$-1$};
\node[yshift=-0.3cm] at (2,0) {$0$};
\node[yshift=-0.3cm] at (4,0) {$1$};
}

\newcommand{\drawgraph}{
\draw (0,0) edge[->]  (4.2,0);
\draw (0,0) edge[->] node[left=0.6cm] {$\mu$}    (0,4.2);

\node[yshift=-0.3cm] at (0,0) {$-1$};
\node[yshift=-0.3cm] at (2,0) {$0$};
\node[yshift=-0.3cm] at (4,0) {$1$};

\node[xshift=-0.2cm,anchor=east] at (0,0) {$-1$};
\node[xshift=-0.2cm,anchor=east] at (0,2) {$0$};
\node[xshift=-0.2cm,anchor=east] at (0,4) {$1$};
}
\tikzstyle{arm1}=[]
\tikzstyle{arm2}=[dotted]
\tikzstyle{arm3} = [thin,decorate,decoration={snake,amplitude=1pt,segment length=2pt}]

\tikzstyle{infinite}=[thin,pattern=crosshatch,opacity=0.15]
\tikzstyle{finite}=[draw=none]

\nipsfinalcopy

\title{Bounded Regret for Finite-Armed Structured Bandits}

\author{
Tor Lattimore \\
Department of Computing Science \\
University of Alberta, Canada \\
\texttt{tlattimo@ualberta.ca} \\
\And
R\'emi Munos \\
INRIA \\
Lille, France$^1$ \\
\texttt{remi.munos@inria.fr} \\
}

\begin{document}

\maketitle
\footnotetext[1]{Current affiliation: Google DeepMind.}

\begin{abstract}
We study a new type of $K$-armed bandit problem where the expected return of one arm 
may depend on the returns of other arms. We present a new algorithm for this general class
of problems and show that under certain circumstances it is possible to achieve finite expected cumulative
regret. We also give problem-dependent lower bounds on the cumulative regret showing that at least in special
cases the new algorithm is nearly optimal.
\end{abstract}

%%%%%%%%%%%%%%%%%%%%%%%%%%%%%%%%%%%%%%%%%%%%%
% INTRODUCTION
%%%%%%%%%%%%%%%%%%%%%%%%%%%%%%%%%%%%%%%%%%%%%
\section{Introduction}

\todot{Fix reference style to be consistent.}

The multi-armed bandit problem is a reinforcement learning problem with $K$ actions. At each time-step a learner must choose an action $i$ after which it receives
a reward distributed with mean $\mu_i$.
The goal is to maximise the cumulative reward. This is perhaps the simplest
setting in which the well-known exploration/exploitation dilemma becomes apparent, with a learner being forced to choose between
exploring arms about which she has little information, and exploiting by choosing the arm that currently appears optimal. 

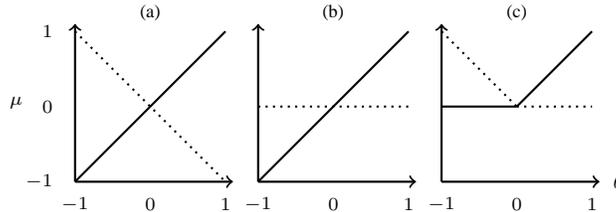
\begin{wrapfigure}[8]{R}{8.4cm}
\vspace{-0.6cm}
\hspace{-0.3cm}
\scriptsize
\begin{tikzpicture}[thick,scale=0.5]
\draw[arm1] (0,0) edge[-] (4,4);
\draw[arm2]  (0,4) edge[-] (4,0);
\node at (2,4.5) {(a)};
\drawgraph
\end{tikzpicture}
\hspace{-0.1cm}
\begin{tikzpicture}[thick,scale=0.5]
\draw[arm1] (0,0) edge[-] (4,4);
\draw[arm2]  (0,2) edge[-] (4,2);
\node at (2,4.5) {(b)};
\drawsimplegraph
\end{tikzpicture}
\hspace{-0.1cm}
\begin{tikzpicture}[thick,scale=0.5]
\draw[finite] (0,0) rectangle (4,4);
\draw[arm1] (0,2) -- (2,2);
\draw[arm2] (0,4) -- (2,2);
\draw[arm1] (2,2) -- (4,4);
\draw[arm2] (2,2) -- (4,2);
\node at (2,4.5) {(c)};
\drawsimplegraph
\node at (4.7,0) {$\theta$};
\end{tikzpicture}
\caption{Examples}
\end{wrapfigure}
We consider a general class of $K$-armed bandit problems where the expected return of each arm may be dependent on other arms.
This model has already been considered when the dependencies are linear \cite{MRPT09} and also in the general setting studied here \cite{GL97,ATA89}. 
Let $\Theta\ni \true$ be an  arbitrary parameter space and define the expected return of arm $i$ 
by $\mu_i(\true) \in \R$. The learner is permitted to know the functions $\mu_1 \cdots \mu_K$, but not the true parameter $\true$.
The unknown parameter $\true$ determines the mean reward for
each arm. The performance of a learner is measured by 
the (expected) cumulative regret, which is the difference between the expected return of the optimal policy and
the (expected) return of the learner's policy.
$R_n \defined n \max_{i\in 1\cdots K} \mu_i(\theta^*) - \sum_{t=1}^n \mu_{I_t}(\theta^*)$
where $I_t$ is the arm chosen at time-step $t$.

A motivating example is as follows. Suppose a long-running company must decide each week whether or not to 
purchase some new form of advertising with unknown expected returns.
The problem may be formulated using the new setting by letting $K = 2$ and $\Theta = [-\infty, \infty]$. We assume the base-line performance
without purchasing the advertising is known and so define $\mu_1(\theta) = 0$ for all $\theta$. 
The expected return of choosing to advertise is $\mu_2(\theta) = \theta$ (see Figure (b) above).

Our main contribution is a new algorithm based on UCB \cite{ACF02} for the structured bandit problem with strong problem-dependent guarantees on the regret. 
The key improvement over UCB is that the new algorithm enjoys finite regret in many cases while UCB suffers logarithmic regret unless all arms have the
same return. 
For example, in (a) and (c) above we show that finite regret is possible for all $\true$, while in the advertising problem finite regret is 
attainable if $\true \geq 0$.
The improved algorithm exploits the known structure and so avoids the famous negative results by Lai and Robbins \cite{LR85}. 
One insight from this work is that knowing the return of the optimal arm and a bound on the minimum gap is not the only information that
leads to the possibility of finite regret. In the examples given above neither quantity is known, but the assumed structure is nevertheless sufficient 
for finite regret. 

Despite the enormous literature on bandits, as far as we are aware this is the first time this setting has been considered with the aim of achieving
finite regret. There has been substantial work on exploiting various kinds of structure to reduce an otherwise impossible problem to one where
sub-linear (or even logarithmic) regret is possible \cite[and references therein]{RV13,AKS11,BMSS08}, but the focus is usually on efficiently
dealing with large action spaces rather than sub-logarithmic/finite regret.
The most comparable previous work studies the case where both the return of the best arm and a bound on the minimum gap between the best arm
and some sub-optimal arm is known \citep{BPR13,BC13}, which extended the permutation bandits studied by Lai and Robbins \cite{LR84b} and more general results
by the same authors \cite{LR84}. 
Also relevant is the paper by Agrawal \etal \cite{ATA89}, which studied a similar setting, but where $\Theta$ was finite. 
Graves and Lai \cite{GL97} extended the aforementioned contribution to continuous parameter spaces (and also to MDPs). Their work differs from ours in a number
of ways. Most notably, their objective is to compute exactly the asymptotically optimal regret in the case where finite regret is {\it not} possible.
In the case where finite regret is possible they prove only that the optimal regret is sub-logarithmic, and do not present any explicit bounds on the actual regret.
Aside from this the results depend on the parameter space being a metric space and they assume that the optimal policy is locally constant about the
true parameter.

%We start by introducing some new notation and formalising our assumptions on the distribution of each arm (Section \ref{sec:notation}).
%We prove a logarithmic bound for the hard case where samples from the best arm are insufficient for the learner
%to discover the optimal arm (Section \ref{sec:ucb}). We then study the case where finite regret is possible (Section \ref{sec:finite}) before presenting some
%lower bounds (Section \ref{sec:lower}) and concluding (Section \ref{sec:conclusions}).
%A table of notation can be found in Appendix \ref{A:notation}.

%%%%%%%%%%%%%%%%%%%%%%%%%%%%%%%%%%%%%%%%%%%%%
% PRELIMINARIES
%%%%%%%%%%%%%%%%%%%%%%%%%%%%%%%%%%%%%%%%%%%%%
\section{Notation}\label{sec:notation}

\subsubsect{General}
Most of our notation is common with \cite{BC12}.
The indicator function is denoted by $\ind{expr}$ and is $1$ if $expr$ is true and $0$ otherwise.
We use $\log$ for the natural logarithm. Logical and/or are denoted by $\wedge$ and $\vee$ respectively.
Define function $\omega(x) = \min \set{y \in \N : z \geq x \log z,\; \forall z \geq y}$, which
satisfies $\log \omega(x) \in O(\log x)$. In fact, $\lim_{x\to\infty} \log(\omega(x)) / \log(x) = 1$.

\subsubsect{Bandits}
Let $\Theta$ be a set.
A $K$-armed structured bandit is characterised by a set of functions $\mu_k:\Theta \to \R$
where $\mu_k(\theta)$ is the expected return of arm $k \in A \defined \set{1, \cdots, K}$ given unknown parameter $\theta$.
We define the mean of the optimal arm by the function $\mu^*:\Theta \to \R$ with $\mu^*(\theta) \defined \max_i \mu_i(\theta)$.
The true unknown parameter that determines the means is $\theta^* \in \Theta$.
The best arm is $\best \defined \argmax_i \mu_i(\theta^*)$.
The arm chosen at time-step $t$ is denoted by $I_t$ while $X_{i,s}$ is the $s$th reward obtained when sampling from arm $i$.
We denote the number of times arm $i$ has been chosen at time-step $t$ by $T_i(t)$.
The empiric estimate of the mean of arm $i$ based on the first $s$ samples is $\hat \mu_{i,s}$. 
We define the gap between the means of the best arm and arm $i$ by $\gap{i} \defined \mu^*(\theta^*) - \mu_i(\theta^*)$.
The set of sub-optimal arms is $A' \defined \set{i \in A : \gap{i} > 0}$. The minimum gap is $\mingap \defined \min_{i \in A'} \gap{i}$ while
the maximum gap is $\maxgap \defined \max_{i \in A} \gap{i}$.
The cumulative regret is defined
\eq{
R_n \defined \sum_{t=1}^n \mu^*(\theta^*) - \sum_{t=1}^n \mu_{I_t} = \sum_{t=1}^n \gap{I_t}
}
Note quantities like $\gap{i}$ and $\best$ depend on $\theta^*$, which is omitted from the notation.
As is rather common we assume that the returns are sub-gaussian, which means that if $X$ is the return sampled from some arm, then
$\ln \E \exp(\lambda(X - \E X)) \leq \lambda^2 \sigma^2 / 2$.
As usual we assume that $\sigma^2$ is known and does not depend on the arm.
If $X_1 \cdots X_n$ are sampled independently from some arm with mean $\mu$ and $S_n = \sum_{t=1}^n X_t$,
then the following maximal concentration inequality is well-known.
\eq{
\P{\max_{1\leq t\leq n} |S_t -t\mu| \geq \epsilon} \leq 2\exp\left(-{\epsilon^2 \over 2n\sigma^2}\right).
}
A straight-forward corollary is that
$\displaystyle \P{|\hat \mu_{i,n} - \mu_i| \geq \epsilon} \leq 2\exp\left(-{\epsilon^2 n \over 2\sigma^2}\right)$.

It is an important point that $\Theta$ is completely arbitrary. The classic multi-armed bandit can be obtained by setting
$\Theta = \R^K$ and $\mu_k(\theta) = \theta_k$, which removes all dependencies between the arms.
The setting where the optimal expected return is known to be zero and a bound on $\gap{i} \geq \epsilon$ is known can
be regained by choosing $\Theta = (-\infty,-\epsilon]^K \times \set{1,\cdots,K}$ and
$\mu_k(\theta_1, \cdots, \theta_K, i) = \theta_k \ind{k \neq i}$.
We do not demand that $\mu_k:\Theta \to \R$ be continuous, or even that $\Theta$ be endowed with a topology.

%%%%%%%%%%%%%%%%%%%%%%%%%%%%%%%%%%%%%%%%%%%%%
% ALGORITHMS
%%%%%%%%%%%%%%%%%%%%%%%%%%%%%%%%%%%%%%%%%%%%%

\section{Structured UCB}\label{sec:ucb}
We propose a new algorithm called UCB-S that is a straight-forward modification of UCB \citep{ACF02}, but where the known structure
of the problem is exploited. At each time-step it constructs a confidence interval about the mean of each arm. From this
a subspace $\tilde\Theta_t \subseteq \Theta$ is constructed, which contains the true parameter $\theta$ with high probability.
The algorithm takes the optimistic action over all $\theta \in \tilde\Theta_t$.

\begin{algorithm}[H]
\caption{UCB-S}
\label{alg:ucbd}
\begin{algorithmic}[1]
\State {\bf Input:} functions $\mu_1, \cdots, \mu_k:\Theta \to [0,1]$
\For{$t \in 1,\ldots,\infty$}
\State Define confidence set
$\displaystyle \tilde\Theta_t \leftarrow \set{\tilde \theta : \forall i,\;\; \left|\mu_i(\tilde \theta) - \hat \mu_{i,T_i(t-1)}\right| < 
\sqrt{{\alpha\sigma^2 \log t \over T_i(t-1)} }}$
\If{$\tilde\Theta_t = \emptyset$}
\State Choose arm arbitrarily
\Else
\State Optimistic arm is $i \leftarrow \argmax_{i} \sup_{\tilde \theta \in \tilde\Theta_t} \mu_i(\tilde \theta)$
\State Choose arm $i$
\EndIf
\EndFor
\end{algorithmic}
\end{algorithm}

\begin{remark}
The choice of arm when $\tilde\Theta_t = \emptyset$ does not affect the regret bounds in this paper.
In practice, it is possible to simply increase $t$ without taking an action, but this complicates the analysis.
In many cases the true parameter $\theta^*$ is never identified in the sense that we do not expect 
that $\tilde\Theta_t \to \set{\theta^*}$.
The computational complexity of UCB-S depends on the difficulty of computing $\tilde\Theta_t$ and computing the optimistic arm within this set.
This is efficient in simple cases, like when $\mu_k$ is piecewise linear, but may be intractable for complex functions. 
\end{remark}

%%%%%%%%%%%%%%%%%%%%%%%%%%%%%%%%%%%%%%%%%%%%%
% Theorems
%%%%%%%%%%%%%%%%%%%%%%%%%%%%%%%%%%%%%%%%%%%%%
\section{Theorems}

We present two main theorems bounding the regret of the UCB-S algorithm. The first is for arbitrary $\theta^*$, which leads to a logarithmic bound
on the regret comparable to that obtained for UCB by \cite{ACF02}.
The analysis is slightly different because UCB-S maintains upper and lower confidence bounds and selects its actions optimistically from the model class,
rather than by maximising the upper confidence bound as UCB does.

\begin{theorem}\label{thm:ucb}
If $\alpha > 2$ and
$\theta \in \Theta$, then the algorithm UCB-S suffers an expected regret of at most
\eq{
\E R_n \leq {2\maxgap K(\alpha - 1) \over \alpha - 2} + \sum_{i \in A'} {8\alpha\sigma^2 \log n \over \gap{i}} + \sum_i \gap{i}
}
\end{theorem}

If the samples from the optimal arm are sufficient to learn the optimal action, then finite regret is possible.
In Section \ref{sec:lower} we give something of a converse by showing that if knowing the mean of the optimal arm
is insufficient to act optimally, then logarithmic regret is unavoidable.

\begin{theorem}\label{thm:finite}
Let $\alpha = 4$ and assume
there exists an $\epsilon > 0$ such that
\eqn{
\label{eq:finite:cond}
(\forall \theta \in \Theta) \qquad \left|\mu_\best(\theta^*) - \mu_\best(\theta)\right| < \epsilon \implies 
\forall i \neq \best,
\mu_\best(\theta) > \mu_i(\theta).
}
Then  
$\displaystyle \E R_n\leq \sum_{i \in A'} \left({32\sigma^2 \log \omega^* \over \gap{i}} + \gap{i}\right) + 3\maxgap K + {\maxgap K^3 \over \omega^*}$, \\
with $\displaystyle \omega^* \defined \max\set{\omega\left({8\sigma^2 \alpha K \over \epsilon^2}\right),\;\omega\left({8\sigma^2 \alpha K \over \mingap^2}\right)}$.
\end{theorem}

\begin{remark}
For small $\epsilon$ and large $n$ the expected regret looks like
$\displaystyle \smash[b]{\E R_n \in O\left(\sum_{i=1}^K {\log \left({1 \over \epsilon}\right) \over \gap{i}}\right)}$ \\ 
(for small $n$ the regret is, of course, even smaller).
\end{remark}

The explanation of the bound is as follows. If at some time-step $t$ it holds that all confidence intervals contain the truth
and the width of the confidence interval about $\best$ drops below $\epsilon$, then
by the condition in Equation (\ref{eq:finite:cond}) it holds that $\best$ is the optimistic arm within $\tilde\Theta_t$. In this
case UCB-S suffers no regret at this time-step. Since the number of samples of each sub-optimal arm grows at most logarithmically 
by the proof of Theorem \ref{thm:ucb}, the 
number of samples of the best arm must grow linearly. Therefore the number of time-steps before best arm has been pulled $O(\epsilon^{-2})$ times
is also $O(\epsilon^{-2})$. After this point the algorithm suffers only a constant cumulative penalty for the possibility that the confidence
intervals do not contain the truth, which is finite for suitably chosen values of $\alpha$. 
Note that Agrawal \etal \cite{ATA89} had essentially the same condition to achieve finite regret as (\ref{eq:finite:cond}), but specified to the case where $\Theta$
is finite.

An interesting question is raised by comparing the bound in Theorem \ref{thm:finite} to those given by Bubeck \etal \cite{BPR13} where if the expected return of the best arm
is known and $\epsilon$ is a known bound on the minimum gap, then a regret bound of 
\eqn{
\label{eq:bubeck}
O\left(\sum_{i \in A'} \left({\log\left({2\gap{i} \over \epsilon}\right) \over \gap{i}}\left(1 + \log\log{1 \over \epsilon}\right) \right)\right)
}
is achieved. If $\epsilon$ is close to $\gap{i}$, then this bound is an improvement over the bound given by Theorem \ref{thm:finite}, although our theorem
is more general.
The improved UCB algorithm \citep{AO10} enjoys a bound on the expected regret of
%$O\left(\sum_{i \in A'} {\log {n\gap{i}^2} \over \gap{i}}\right)$.
$O(\sum_{i \in A'} {1 \over \gap{i}} \log {n\gap{i}^2})$.
If we follow the same reasoning as above we obtain a bound comparable to (\ref{eq:bubeck}). Unfortunately though, the extension of
the improved UCB algorithm to the structured setting is rather challenging with the main obstruction being the extreme growth of the phases
used by improved UCB. Refining the phases leads to super-logarithmic regret, a problem we ultimately failed to resolve. 
Nevertheless we feel that there is some hope of obtaining
a bound like (\ref{eq:bubeck}) in this setting.

Before the proofs of Theorems \ref{thm:ucb} and \ref{thm:finite} we give some example structured bandits and indicate the regions where
the conditions for Theorem \ref{thm:finite} are (not) met.
Areas where Theorem \ref{thm:finite} can be applied to obtain finite regret are unshaded while those with logarithmic regret are shaded.
%%%%%%%%%%%%%%%%%%%%%%%%%%%%%%%%%%%%%%%%%%%%%
% EXAMPLES
%%%%%%%%%%%%%%%%%%%%%%%%%%%%%%%%%%%%%%%%%%%%%

\begin{figure}[H]
\centering
\begin{tikzpicture}[thick,scale=0.5]
\draw[finite] (0,0) rectangle (4,4);
\draw[arm1] (0,0) edge[-] (4,4);
\draw[arm2]  (0,4) edge[-] (4,0);
\drawgraph
\node at (2,4.5) {(a)};
\end{tikzpicture}
\hspace{0.5cm}
\begin{tikzpicture}[thick,scale=0.5]
\draw[finite] (2,0) rectangle (4,4);
\draw[infinite] (0,0) rectangle (2,4);
\draw[arm1] (0,0) edge[-] (4,4);
\draw[arm2]  (0,2) edge[-] (4,2);
\drawsimplegraph
\node at (2,4.5) {(b)};
\end{tikzpicture}
\hspace{0.5cm}
\begin{tikzpicture}[thick,scale=0.5]
\draw[finite] (0,0) rectangle (4,4);
\draw[arm1] (0,2) -- (2,2);
\draw[arm2] (0,4) -- (2,2);
\draw[arm1] (2,2) -- (4,4);
\draw[arm2] (2,2) -- (4,2);
\drawsimplegraph
\node at (2,4.5) {(c)};
\node at (4.7,0) {$\theta$};
\end{tikzpicture}
\hspace{0.5cm}
\begin{tikzpicture}[thick,scale=0.5]
\draw[thin] (-1,-1) rectangle (3,4);
\draw (-0.5,2) edge[arm1] node[right=0.5cm] {$\mu_1$} (1,2);
\draw (-0.5,1) edge[arm2] node[right=0.5cm] {$\mu_2$} (1,1);
\draw (-0.5,0) edge[arm3] node[right=0.5cm] {$\mu_3$} (1,0);
\node at (0,3.3) {{\scriptsize \bf Key:}};
\node at (2,-1.5) {\color{white}a hidden message};
\end{tikzpicture}

\hspace{-1.2cm}
\begin{tikzpicture}[thick,scale=0.5]
\draw[infinite] (0,0) rectangle (3,4);
\draw[finite] (3,0) rectangle (4,4);
\draw[arm1] (0,3) edge (1,3);
\draw[arm1] (1,1) edge (3,1);
\draw[arm1] (3,3) edge (4,4);
\draw[arm2] (0,3) edge (1,2);
\draw[arm2] (1,2) edge (3,2);
\draw[arm2] (3,4) edge (4,4);
\drawgraph
\node at (2,4.5) {(d)};
\end{tikzpicture}
\hspace{0.5cm}
\begin{tikzpicture}[thick,scale=0.5]
\draw[finite] (0,0) rectangle (1,4);
\draw[infinite]   (1,0) rectangle (2,4);
\draw[finite] (2,0) rectangle (4,4);
\draw[infinite] (3,0) -- (3,4);
\draw[arm1] (1,1) edge[-] (4,4);
\draw[arm2]  (1,2) edge[-] (4,2);
\draw[arm1] (0,3) edge[-] (1,3);
\draw[arm2] (0,4) edge[-] (1,4);
\node at (2,4.5) {(e)};
\drawsimplegraph
\end{tikzpicture}
\hspace{0.5cm}
\begin{tikzpicture}[thick,scale=0.5]
\draw[finite] (0,0) rectangle (4,4);
\draw[arm1] (0,1) edge (1.5,1);
\draw[arm2] (0,2) edge (1.5,2);
\draw[arm3] (0.1,3) edge[arm3] (3,3);
\draw[arm2] (1.5,1) edge (4.5,1);
\draw[arm1] (1.5,2) edge (3,2);
\draw[arm3] (3,2) edge[arm3] (4.5,2);
\draw[arm1] (3,3) edge (6,3);
\draw[arm2] (4.5,2) edge (6,2);
\draw[arm3] (4.5,1) edge[arm3] (7.5,1);
\draw[arm2] (6,3) edge (9,3);
\draw[arm1] (6,2) edge (7.5,2);
\draw[arm1] (7.5,1) edge (9,1);
\draw[arm3] (7.5,2) edge[arm3] (9,2);

\draw (0,0) edge[->]  (9.1,0);
\draw (0,0) edge[->]  (0,4.2);
\node at (9.7,0) {$\theta$};
\node at (4.5,4.5) {(f)};

\node[yshift=-0.3cm] at (0,0) {\color{white}$-1$};
\node[yshift=-0.3cm] at (0.75,0) {$1$};
\node[yshift=-0.3cm] at (2.25,0) {$2$};
\node[yshift=-0.3cm] at (3.75,0) {$3$};
\node[yshift=-0.3cm] at (5.25,0) {$4$};
\node[yshift=-0.3cm] at (6.75,0) {$5$};
\node[yshift=-0.3cm] at (8.25,0) {$6$};
\end{tikzpicture}
\caption{Examples}\label{fig:examples}
\end{figure}
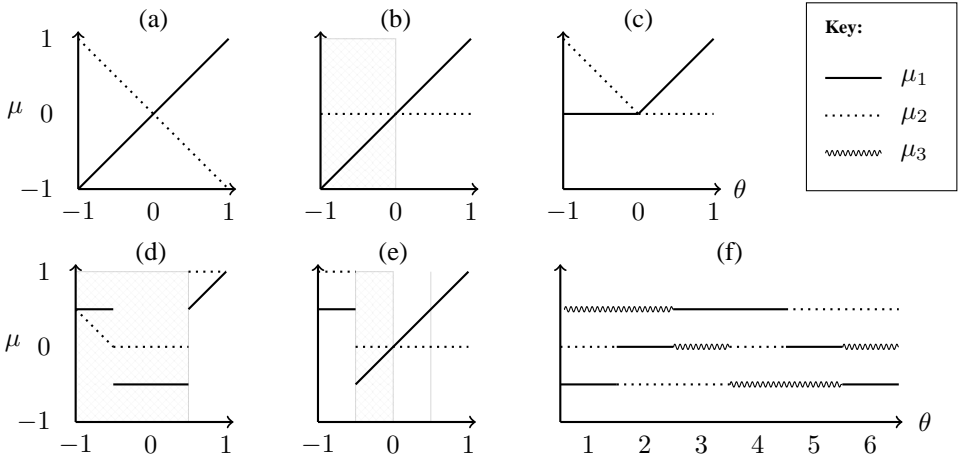

\begin{enumerate}[(a)]
\item The conditions for Theorem \ref{thm:finite} are met for all $\theta \neq 0$, but for $\theta = 0$ the regret strictly vanishes
for all policies, which means that the regret is bounded by 
$\E R_n \in O(\ind{\theta^* \neq 0} {1 \over |\theta^*|} \log{1 \over |\theta^*|})$.
\item Action 2 is uninformative and not globally optimal so Theorem \ref{thm:finite} does not apply for $\theta < 1/2$ where this action is optimal.
For $\theta > 0$ the optimal action is 1, when the conditions are met and finite regret is again achieved.
\eq{
\E R_n \in 
O\left(\ind{\theta^* < 0}{\log n \over |\theta^*|} + \ind{\theta^* > 0}  {\log{1 \over \theta^*} \over \theta^*}\right).
}
\item The conditions for Theorem \ref{thm:finite} are again met for all non-zero $\true$, which leads as in (a) to a regret of
$\E R_n \in O(\ind{\theta^* \neq 0} {1 \over |\theta^*|} \log{1 \over |\theta^*|})$.
%The bound can be improved to $R_n \in O({\ind{\true \neq 0} \over |\true|})$ by using a different algorithm and rather technical analysis, which is omitted.
%for details on this curious problem, which is included because the situation is superficially similar to the one covered by Theorem 1 in \citep{BPR13}, but here neither
%the gap nor optimal return are known. Despite this, we achieve the same bound on the regret. 
\end{enumerate}
Examples (d) and (e) illustrate the potential complexity of the regions in which finite regret is possible. 
Note especially that in (e) the regret for $\true = {1 \over 2}$ is logarithmic in the horizon, but finite for $\true$ arbitrarily close.
Example (f) is a permutation bandit with 3 arms
where it can be clearly seen that the conditions of Theorem \ref{thm:finite} are satisfied. 

%%%%%%%%%%%%%%%%%%%%%%%%%%%%%%%%%%%%%%%%%%%%%
% UCB Bound
%%%%%%%%%%%%%%%%%%%%%%%%%%%%%%%%%%%%%%%%%%%%%
\section{Proof of Theorems \ref{thm:ucb} and \ref{thm:finite}}\label{sec:ucb}

We start by bounding the probability that some mean does not lie inside the confidence set.
\begin{lemma}\label{lem:failure}
$\P{F_t = 1} \leq 2Kt \exp(-\alpha \log(t))$ where
\eq{
F_t = \ind{\exists i : |\hat \mu_{i,T_i(t-1)} - \mu_i| \geq \sqrt{2\alpha \sigma^2 \log t \over T_i(t-1)}}.
}
\end{lemma}

\begin{proof}
We use the concentration guarantees:
\eq{
&\P{F_t = 1} 
\sr{(a)}= \P{\exists i : \left|\mu_i(\theta^*) - \hat \mu_{i,T_i(t-1)}\right| \geq \sqrt{2\alpha \sigma^2 \log t\over T_i(t - 1)}} \\
&\sr{(b)}\leq \sum_{i=1}^K \P{\left|\mu_i(\theta^*) - \hat \mu_{i,T_i(t-1)}\right| \geq \sqrt{2\alpha\sigma^2 \log t\over T_i(t - 1)}} \\
&\sr{(c)}\leq\sum_{i=1}^K \sum_{s=1}^t \P{\left|\mu_i(\theta^*) - \hat \mu_{i,s}\right| \geq \sqrt{2\alpha \sigma^2 \log t \over s}}
\sr{(d)}\leq\sum_{i=1}^K \sum_{s=1}^t 2 \exp(-\alpha \log t) 
\sr{(e)}= 2Kt^{1-\alpha}
}
where (a) follows from the definition of $F_t$.
(b) by the union bound.
(c) also follows from the union bound and is the standard trick to deal with the random variable $T_i(t-1)$. 
(d) follows from the concentration inequalities for sub-gaussian random variables.
(e) is trivial.
\end{proof}

\begin{proof}[Proof of Theorem \ref{thm:ucb}]
Let $i$ be an arm with $\gap{i} > 0$ and suppose that $I_t = i$. Then either $F_t$ is true or
\eqn{
\label{eq:ucb:1} T_i(t-1) &< \ceil{{8\sigma^2 \alpha \log n \over \gap{i}^2}} \rdefined u_i(n) 
}
Note that if $F_t$ does not hold then the true parameter lies within the confidence set, $\true \in \tilde \Theta_t$.
Suppose on the contrary that $F_t$ and (\ref{eq:ucb:1}) are both false. 
\eq{
\max_{\tilde\theta \in \tilde\Theta_t} \mu_{\best}(\tilde\theta)  
&\sr{(a)}\geq \mu^*(\true) 
\sr{(b)}=\mu_i(\true) + \gap{i} 
\sr{(c)}> \gap{i} + \hat \mu_{i,T_i(t-1)} - \sqrt{2\sigma^2 \alpha \log t \over T_i(t-1)}  \\
&\sr{(d)}\geq \hat \mu_{i,T_i(t-1)} + \sqrt{2\alpha\sigma^2 \log t \over T_i(t-1)} 
\sr{(e)}\geq \max_{\tilde\theta \in \tilde\Theta_t} \mu_i(\tilde\theta),
}
where (a) follows since $\theta^* \in \tilde\Theta_t$.
(b) is the definition of the gap.
(c) since $F_t$ is false.
(d) is true because (\ref{eq:ucb:1}) is false.
Therefore arm $i$ is not taken. We now bound the expected number of times that arm $i$ is played within the first $n$ time-steps by
\eq{
\E T_i(n) 
&\sr{(a)}= \E \sum_{t=1}^n \ind{I_t = i}  
\sr{(b)}\leq u_i(n) + \E \sum_{t=u_i+1}^n \ind{I_t = i \wedge (\ref{eq:ucb:1})\text{ is false}} \\
&\sr{(c)}\leq u_i(n) + \E \sum_{t=u_i+1}^n \ind{F_t = 1 \wedge I_t = i}
}
where (a) follows from the linearity of expectation and definition of $T_i(n)$. 
(b) by Equation (\ref{eq:ucb:1}) and the definition of $u_i(n)$ and expectation. 
(c) is true by recalling that playing arm $i$ at time-step $t$ implies that either $F_t$ or (\ref{eq:ucb:1}) must be true.
Therefore
\eqn{
\E R_n 
&\leq \sum_{i \in A'} \gap{i} \left(u_i(n) + \E \sum_{t=u_i+1}^n \ind{F_t = 1 \wedge I_t = i}\right) 
\label{eq:ucb:2} \leq \sum_{i \in A'} \gap{i} u_i(n) + \maxgap \E \sum_{t=1}^n \ind{F_t = 1}
}
Bounding the second summation
\eq{
\E \sum_{t=1}^n \ind{F_t = 1} 
\sr{(a)}= \sum_{t=1}^n \P{F_t = 1} 
\sr{(b)}\leq \sum_{t=1}^n 2Kt^{1-\alpha} 
\sr{(c)}\leq {2K(\alpha - 1) \over \alpha - 2}
}
where (a) follows by exchanging the expectation and sum and because the expectation of an indicator function can be written as the probability of the event.
(b) by Lemma \ref{lem:failure} and (c) is trivial.
Substituting into (\ref{eq:ucb:2}) leads to
\eq{
\E R_n \leq {2\maxgap K(\alpha - 1) \over \alpha - 2} + \sum_{i \in A'} {8\alpha\sigma^2 \log n \over \gap{i}} + \sum_i \gap{i}.
}
\end{proof}

%%%%%%%%%%%%%%%%%%%%%%%%%%%%%%%%%%%%%%%%%%%%%
% Finite Regret
%%%%%%%%%%%%%%%%%%%%%%%%%%%%%%%%%%%%%%%%%%%%%
%\section{Proof of Theorem \ref{thm:finite}}\label{sec:finite}

Before the proof of Theorem \ref{thm:finite} we need a high-probability bound on the number of times arm $i$ is pulled, which is
proven along the lines of similar results by \cite{AMS07}.

\begin{lemma}\label{lem:high-prob}
Let $i \in A'$ be some sub-optimal arm.
If $z > u_i(n)$, then
$\displaystyle \P{T_i(n) > z} \leq {2K z^{2-\alpha}\over \alpha - 2}$.
\end{lemma}

\begin{proof}
As in the proof of Theorem \ref{thm:ucb}, if $t \leq n$ and $F_t$ is false and $T_i(t-1) > u_i(n) \geq u_i(t)$, then arm $i$ is not chosen.
Therefore
\eq{
\P{T_i(n) > z} 
\leq \sum_{t=z+1}^n \P{F_t = 1} 
\sr{(a)}\leq \sum_{t=z+1}^n 2K t^{1-\alpha} 
\sr{(b)}\leq 2K \int^n_{z}  t^{1 - \alpha} dt 
\sr{(c)}\leq {2K z^{2-\alpha} \over \alpha - 2}  
}
where (a) follows from Lemma \ref{lem:failure} and (b) and (c) are trivial.
\end{proof}

\begin{lemma}
Assume the conditions of Theorem \ref{thm:finite} and additionally that
$T_\best(t - 1) \geq \ceil{8\alpha\sigma^2 \log t \over \epsilon^2}$ and $F_t$ is false. Then $I_t = \best$.
\end{lemma}

\begin{proof}
Since $F_t$ is false, for $\tilde \theta \in \tilde\Theta_t$ we have:
\eq{
|\mu_\best(\tilde\theta) - \mu_\best(\theta^*)| 
&\sr{(a)}\leq |\mu_\best(\tilde\theta) - \hat\mu_{\best,T_i(t-1)}| + |\hat\mu_{\best,T_i(t-1)} - \mu_\best(\theta^*)| 
\sr{(b)}< 2\sqrt{2\sigma^2 \alpha \log t \over T_\best(t-1)} 
\sr{(c)}\leq \epsilon
}
where (a) is the triangle inequality. (b) follows by the definition of the confidence interval and because $F_t$ is false.
(c) by the assumed lower bound on $T_\best(t-1)$.
Therefore by (\ref{eq:finite:cond}), for all $\tilde\theta \in \tilde\Theta_t$ it holds 
that the best arm is $\best$. Finally, since $F_t$ is false, $\theta^* \in \tilde \Theta_t$, which means that $\tilde\Theta_t \neq \emptyset$.
Therefore $I_t = \best$ as required.
\end{proof}

\begin{proof}[Proof of Theorem \ref{thm:finite}]
Let $\omega^*$ be some constant to be chosen later. Then the regret may be written as
\eqn{
\label{eq:ucb:4}
\E R_n \leq \E \sum_{t=1}^{\omega^*} \sum_{i=1}^K \gap{i} \ind{I_t = i} + \maxgap \E \sum_{t=\omega^*+1}^n \ind{I_t \neq \best}.
}
The first summation is bounded as in the proof of Theorem \ref{thm:ucb} by
\eqn{
\label{eq:ucb:5}
\E \sum_{t=1}^{\omega^*} \sum_{i \in A} \gap{i} \ind{I_t = i} 
&\leq \sum_{i \in A'} \left(\gap{i} + {8\alpha\sigma^2 \log \omega^* \over \gap{i}} \right)
+ \sum_{t=1}^{\omega^*} \P{F_t = 1}.
}
We now bound the second sum in (\ref{eq:ucb:4}) and choose $\omega^*$.
By Lemma \ref{lem:high-prob}, if ${n \over K} > u_i(n)$, then
\eqn{
\label{eq:ucb:3}
\P{T_i(n) > {n \over K}} \leq {2K \over \alpha - 2} \left({K \over n}\right)^{\alpha - 2}.
}
Suppose 
$t \geq \omega^* \defined \max\set{\omega\left({8\sigma^2 \alpha K \over \epsilon^2}\right),\;\omega\left({8\sigma^2 \alpha K \over \mingap^2}\right)}$.
Then ${t\over K} > u_i(t)$ for all $i \neq i^*$ and ${t \over K} \geq {8\sigma^2 \alpha \log t \over \epsilon^2}$. By the union bound
\eqn{
\P{T_\best(t) < {8\sigma ^2\alpha \log t \over \epsilon^2}} 
&\sr{(a)}\leq \P{T_\best(t) < {t \over K}} 
\sr{(b)}\leq \P{\exists i : T_i(t) > {t \over K}} 
\label{eq:ucb:7} \sr{(c)}< {2K^2 \over \alpha - 2} \left({K \over t}\right)^{\alpha - 2}
}
where (a) is true since ${t \over K} \geq {8\sigma^2\alpha \log t \over \epsilon^2}$.
(b) since $\sum_{i=1}^K T_i(t) = t$.
(c) by the union bound and (\ref{eq:ucb:3}).
Now if $T_i(t) \geq {8\sigma^2\alpha \log t \over \epsilon^2}$ and $F_t$ is false, then the chosen arm is $\best$.
Therefore
\eqn{
\nonumber \E \sum_{t=\omega^*+1}^n \ind{I_t \neq i^*}
&\leq \sum_{t=\omega^*+1}^n \P{F_t = 1} + \sum_{t=\omega^*+1}^n \P{T_i(t-1) < {8\sigma^2 \alpha \log t \over \epsilon^2}} \\
\nonumber &\sr{(a)}\leq \sum_{t=\omega^*+1}^n \P{F_t = 1} + {2K^2 \over \alpha - 2} \sum_{t=\omega^*+1}^n \left({K \over t}\right)^{\alpha - 2} \\
\label{eq:ucb:6} &\sr{(b)}\leq \sum_{t=\omega^*+1}^n \P{F_t = 1} + {2K^2 \over (\alpha - 2)(\alpha - 3)}\left({K \over \omega^*}\right)^{\alpha-3} 
}
where (a) follows from (\ref{eq:ucb:7}) and (b) by straight-forward calculus.
Therefore by combining (\ref{eq:ucb:4}), (\ref{eq:ucb:5}) and (\ref{eq:ucb:6}) we obtain
\eq{
\E R_n 
&\leq \sum_{i: \gap{i} > 0} \gap{i} \ceil{8\sigma^2 \alpha \log \omega^* \over \gap{i}^2} + {2\maxgap K^2 \over (\alpha - 2)(\alpha - 3)}\left({K\over \omega^*}\right)^{\alpha-3}
+ \maxgap \sum_{t=1}^n \P{F_t = 1} \\
&\leq \sum_{i: \gap{i} > 0} \gap{i} \ceil{8\sigma^2 \alpha \log \omega^* \over \gap{i}^2} + {2\maxgap K^2 \over (\alpha - 2)(\alpha - 3)}\left({K\over \omega^*}\right)^{\alpha-3}
+ {2\maxgap K(\alpha - 1) \over \alpha - 2}
}
Setting $\alpha = 4$ leads to
$\displaystyle \E R_n\leq \sum_{i=1}^K \left({32\sigma^2 \log \omega^* \over \gap{i}} + \gap{i}\right) + 3\maxgap K + {\maxgap K^3 \over \omega^*}$.
\end{proof}

%%%%%%%%%%%%%%%%%%%%%%%%%%%%%%%%%%%%%%%%%%%%%
% LOWER BOUNDS
%%%%%%%%%%%%%%%%%%%%%%%%%%%%%%%%%%%%%%%%%%%%%
\section{Lower Bounds and Ambiguous Examples}\label{sec:lower}

We prove lower bounds for two illustrative examples of structured bandits.
Some previous work is also relevant. 
The famous paper by Lai and Robbins \cite{LR85} shows that the bound of Theorem \ref{thm:ucb} cannot
in general be greatly improved. Many of the techniques here are borrowed from Bubeck \etal \cite{BPR13}.
Given a fixed algorithm and varying $\theta$ we denote the regret and expectation by $R_n(\theta)$ and $\E_\theta$ respectively.
Returns are assumed to be sampled from a normal distribution with unit variance, so that $\sigma^2 = 1$. \ifsup \else The proofs of the following theorems
may be found in the supplementary material. \fi

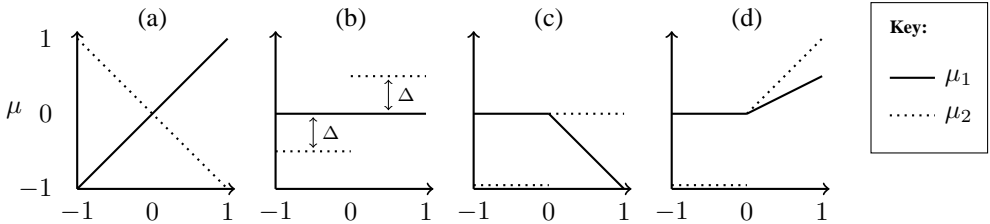
\begin{figure}[H]
\centering
\begin{tikzpicture}[thick,scale=0.5]
\draw[arm1] (0,0) -- (4,4);
\draw[arm2] (0,4) -- (4,0);
\node at (2,4.5) {(a)};
\drawgraph{}
\end{tikzpicture}
\begin{tikzpicture}[thick,scale=0.5]
\draw[arm1] (0,2) edge[-] (4,2);
\draw[arm2]  (0,1) edge[-] (2,1);
\draw[arm2]  (2,3) edge[-] (4,3);
\draw[very thin] (1,1.1) edge[<->] node[right] {\scriptsize $\Delta$} (1,1.9);
\draw[very thin] (3,2.1) edge[<->] node[right] {\scriptsize $\Delta$} (3,2.9);
\drawsimplegraph
\node at (2,4.5) {(b)};
\end{tikzpicture}
\begin{tikzpicture}[thick,scale=0.5]
\draw[arm1] (0,2) edge[-] (2,2);
\draw[arm1] (2,2) edge[-] (4,0);
\draw[arm2]  (0,0.1) edge[-] (2,0.1);
\draw[arm2]  (2,2) edge[-] (4,2);
\node at (2,4.5) {(c)};
\drawsimplegraph
\end{tikzpicture}
\begin{tikzpicture}[thick,scale=0.5]
\draw[arm1] (0,2) edge[-] (2,2);
\draw[arm1] (2,2) edge[-] (4,3);
\draw[arm2]  (0,0.1) edge[-] (2,0.1);
\draw[arm2]  (2,2) edge[-] (4,4);
%\draw[arm2]  (2,2) edge[-] (4,2);
\node at (2,4.5) {(d)};
\drawsimplegraph
\end{tikzpicture}
\begin{tikzpicture}[thick,scale=0.5]
\draw[thin] (-1,0) rectangle (2.1,4);
\draw (-0.5,2) edge[arm1] node[right=0.3cm] {$\mu_1$} (0.7,2);
\draw (-0.5,1) edge[arm2] node[right=0.3cm] {$\mu_2$} (0.7,1);
\node at (0,3.3) {{\scriptsize \bf Key:}};
\node at (1,-1.5) {\color{white}a hidden message};
\end{tikzpicture}

\caption{Counter-examples}\label{fig:ce}
\end{figure}

\begin{theorem}
Given the structured bandit depicted in Figure \ref{fig:ce}.(a) or Figure \ref{fig:examples}.(c), then for all $\theta > 0$ and all algorithms the
regret satisfies
$\max\set{\E_{-\theta} R_n(-\theta),\; \E_{\theta} R_n(\theta)} \geq {1 \over 8\theta}$
for sufficiently large $n$.
\end{theorem}

\newcommand{\Po}[2]{\mathbb P_{#1}\left\{#2\right\}}
\ifsup
\begin{proof}
The proof uses the same technique as the proof of Theorem 5 in the paper by \cite{BPR13}.
Fix an algorithm and 
let $\mathbb P_{\theta,t}$ be the probability measure on the space of outcomes up to time-step $t$ under the 
bandit determined by parameter $\theta$.
\eq{
\E_{-\theta} R_n(-\theta) &+ \E_{\theta} R_n(\theta) 
\sr{(a)}= 2\theta \left(\E_{-\theta} \sum_{t=1}^n \ind{I_t = 1} + \E_{\theta} \sum_{t=1}^n \ind{I_t = 2}\right) \\
&\sr{(b)}= 2\theta \sum_{t=1}^n \left(\Po{-\theta,t}{I_t = 1} + \Po{\theta,t}{I_t = 2}\right) 
\sr{(c)}\geq \theta \sum_{t=1}^n \exp\left(-\KL(\mathbb P_{-\theta,t}, \mathbb P_{\theta,t})\right) \\ 
&\sr{(d)}= \theta \sum_{t=1}^n \exp\left(-4t\theta^2\right) 
\sr{(e)}\geq {1 \over 8\theta} 
}
where (a) follows since $2|\theta|$ is the gap between the expected returns of the two arms given parameter $\theta$ and by
the definition of the regret.
(b) by replacing the expectations with probabilities.
(c) follows from Lemma 4 by \cite{BPR13} where $\KL(\mathbb P_{-\theta,t},\mathbb P_{\theta,t})$ is the relative entropy between measures $\mathbb P_{-\theta,t}$ and $\mathbb P_{\theta,t}$.
(d) is true by computing the relative entropy between two normals with unit variance and means separated by $2\theta$, which is $4\theta^2$.
(e) holds for sufficiently large $n$.
\end{proof}
\fi

\begin{theorem}\label{thm:lower}
Let $\Theta, \set{\mu_1,\mu_2}$ be a structured bandit where returns are sampled from a normal distribution with unit variance.
Assume there exists a pair $\theta_1, \theta_2 \in \Theta$ and constant $\Delta > 0$ 
such that $\mu_1(\theta_1) = \mu_1(\theta_2)$ and $\mu_1(\theta_1) \geq \mu_2(\theta_1) + \Delta$ and $\mu_2(\theta_2) \geq \mu_1(\theta_2) + \Delta$.
Then the following hold:
\begin{enumerate}[(1)]
\item $\E_{\theta_1} R_n(\theta_1) \geq {1 + \log {2n \Delta^2} \over 8\Delta} - {1 \over 2} \E_{\theta_2} R_n(\theta_2)$
\item $\E_{\theta_2} R_n(\theta_2) \geq {n\Delta \over 2} \exp\left(-4\E_{\theta_1} R_n(\theta_1)\Delta\right) - \E_{\theta_1} R_n(\theta_1)$
\end{enumerate}
\end{theorem}

A natural example where the conditions are satisfied is depicted in Figure \ref{fig:ce}.(b) and by choosing $\theta_1 = -1$, $\theta_2 = 1$.
We know from Theorem \ref{thm:finite} that UCB-S enjoys finite regret of $\E_{\theta_2} R_n(\theta_2) \in O({1\over\Delta} \log{1 \over \Delta})$ and logarithmic 
regret $\E_{\theta_1} R_n(\theta_1) \in O({{1 \over \Delta} \log n})$.
Part 1 of Theorem \ref{thm:lower} shows that if we demand finite regret $\E_{\theta_2} R_n(\theta_2) \in O(1)$, then 
the regret $\E_{\theta_1} R_n(\theta_1)$ is necessarily logarithmic.
On the other hand, part 2 shows that if we demand $\E_{\theta_1} R_n(\theta_1) \in o(\log(n))$, then the regret $\E_{\theta_2} R_n(\theta_2) \in \Omega(n)$.
Therefore the trade-off made by UCB-S essentially cannot be improved. 

\ifsup
\begin{proof}[Proof of Theorem \ref{thm:lower}]
\todot{Fix the notation in this proof}
Again, we make use of the techniques of \cite{BPR13}.
\eqn{
\nonumber \E_{\theta_1} R_n(\theta_1) + \E_{\theta_2} R_n(\theta_2)
&\sr{(a)}\geq \Delta \left(\E_{\theta_1} T_2(n) + \E_{\theta_2} T_1(n)\right) 
\sr{(b)}\geq \Delta \sum_{t=1}^n \left(\Po{\theta_1}{I_t = 2} + \Po{\theta_2}{I_t = 1}\right) \\
\nonumber &\sr{(c)}\geq {\Delta \over 2} \sum_{t=1}^n \exp\left(-\KL(\mathbb P_{\theta_1,t}, \mathbb P_{\theta_2,t})\right) 
\sr{(d)}\geq {n\Delta \over 2}  \exp\left(-\KL(\mathbb P_{\theta_1,n}, \mathbb P_{\theta_2,n})\right) \\
&\sr{(e)}\geq {n\Delta \over 2} \exp\left(-4\Delta^2 \E_{\theta_1} T_2(n)\right) 
\label{eq:lower}\sr{(f)}\geq {n\Delta \over 2} \exp\left(-4\Delta \E_{\theta_1} R_n(\theta_1) \right)
}
where (a) follows from the definition of the regret and the bandits used.
(b) by the definition of $T_k(n)$.
(c) by Lemma 4 of \cite{BPR13}.
(d) since the relative entropy $\KL(\mathbb P_{\theta_1,t}, \mathbb P_{\theta_2,t})$ is increasing with $t$.
(e) By checking that $\KL(\mathbb P_{\theta_1,n}, \mathbb P_{\theta_2,n}) = 4\Delta^2 \E_{\theta_1} T_2(n)$.
(f) by substituting the definition of the regret.
Now part 2 is completed by rearranging (\ref{eq:lower}).
For part 1 we also rearrange (\ref{eq:lower}) to obtain
\eq{
\E_{\theta_1} R_n(\theta_1) \geq {n\Delta \over 2} \exp\left(-4\Delta \E_{\theta_1} R_n(\theta_1)\right) - \E_{\theta_2} R_n(\theta_2)
}
Letting $x = \E_{\theta_1} R_n(\theta_1)$ and using the constraint above we obtain: 
\eq{
x \geq {x \over 2} + {1 \over 2} \left({n\Delta \over 2}\exp\left(-4\Delta x\right) - \E_{\theta_2} R_n(\theta_2)\right).
}
But by simple calculus the function on the right hand side is minimised for $x = {1 \over 4\Delta} \log (2n\Delta^2)$, which leads to
\eq{
\E_{\theta_1} R_n(\theta_1) \geq {\log (2n\Delta^2) \over 8\Delta} + {1 \over 8\Delta} - {1 \over 2}\E_{\theta_2} R_n(\theta_2).
}
\end{proof}
\fi

\subsubsect{Discussion of Figure \ref{fig:ce}.(c/d)}
In both examples there is an ambiguous region for which the lower bound (Theorem \ref{thm:lower}) does not show that logarithmic regret is unavoidable, but where
Theorem \ref{thm:finite} cannot be applied to show that UCB-S achieves finite regret. We managed to show that finite regret is
possible in both cases by using a different algorithm. For (c) we could construct a carefully tuned algorithm for which the regret
was at most $O(1)$ if $\theta \leq 0$ and $O({1 \over \theta} \log \log{1 \over \theta})$ otherwise. This result contradicts
a claim by Bubeck et.\ al.\ \cite[Thm. 8]{BPR13}. Additional discussion of the ambiguous case in general, as well as this specific example, may 
be found in the supplementary material. One observation is that unbridled optimism is the cause of the failure of
UCB-S in these cases. This is illustrated by Figure \ref{fig:ce}.(d) with $\theta \leq 0$. No matter how narrow the confidence interval about $\mu_1$,
if the second action has not been taken sufficiently often, then there will still be some belief that $\theta > 0$ is possible where the second action is optimistic,
which leads to logarithmic regret. Adapting the algorithm to be slightly risk averse solves this problem.

%%%%%%%%%%%%%%%%%%%%%%%%%%%%%%%%%%%%%%%%%%%%%
% EXPERIMENTS
%%%%%%%%%%%%%%%%%%%%%%%%%%%%%%%%%%%%%%%%%%%%%
\section{Experiments}\label{sec:experiments}
We tested Algorithm \ref{alg:ucbd} on a selection of structured bandits depicted in Figure \ref{fig:examples} and compared to UCB \cite{ACF02,BC12}.
Rewards were sampled from normal distributions with unit variances. For UCB we chose $\alpha = 2$, while we used the theoretically
justified $\alpha = 4$ for Algorithm \ref{alg:ucbd}. All code is available in the supplementary material. Each data-point is the
average of 500 independent samples with the blue crosses and red squares indicating the regret of UCB-S and UCB respectively.

{
{\centering
\scriptsize
\begin{minipage}{4.6cm}
\begin{tikzpicture}[baseline,font=\scriptsize]
  \begin{axis}[ xlabel={$\theta$},
                ylabel={$\hat \E_\theta R_n(\theta)$},
                xmin=-0.21,
                xmax=0.21,
                height=3cm,
                width=4.5cm,
                mark size=0.5pt,
                compat=newest]
    \addplot+[only marks,mark=x,mark size=1.5pt] table[x index=0,y index=1] {exp-a-1.txt};
    \addplot+[only marks] table[x index=0,y index=2] {exp-a-1.txt};
  \end{axis}
\end{tikzpicture} \\
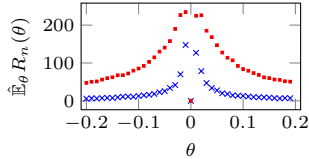
$K = 2$, $\mu_1(\theta) = \theta$, $\mu_2(\theta) = -\theta$, \\ 
$n = 50\,000$ 
(see Figure \ref{fig:examples}.(a)) 
\end{minipage}
\begin{minipage}{4.6cm}
\begin{tikzpicture}[baseline,font=\scriptsize]
  \begin{axis}[ xlabel={$n$},
                ylabel={$\hat \E_\theta R_n(\theta)$},
                height=3cm,
                width=4.5cm,
                mark size=0.5pt,
                scaled ticks=false,
                compat=newest,
                xtick={0,50000,100000},
                xticklabels={0,5e4,1e5}]
    \addplot+[only marks,mark=x,mark size=1.5pt] table[x index=0,y index=1] {exp-a-2.txt};
    \addplot+[only marks] table[x index=0,y index=2] {exp-a-2.txt};
  \end{axis}
\end{tikzpicture} \\
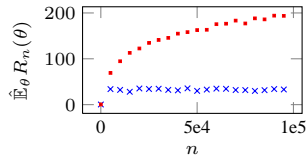
$K = 2$, $\mu_1(\theta) = \theta$, $\mu_2(\theta) = -\theta$, \\
$\theta = 0.04$
(see Figure \ref{fig:examples}.(a)) 
\end{minipage}
\begin{minipage}{4.6cm}
\begin{tikzpicture}[baseline,font=\scriptsize]
  \begin{axis}[ xlabel={$\theta$},
                ylabel={$\hat \E_\theta R_n(\theta)$},
                height=3cm,
                width=4.5cm,
                mark size=0.5pt,
                scaled ticks=false,
                compat=newest]
    \addplot+[only marks,mark=x,mark size=1.5pt] table[x index=0,y index=1] {exp-b-1.txt};
    \addplot+[only marks] table[x index=0,y index=2] {exp-b-1.txt};
  \end{axis}
\end{tikzpicture} \\
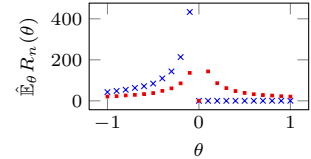
$K = 2$, $\mu_1(\theta) = 0$, $\mu_2(\theta) = \theta$, \\
$n = 50\,000$
(see Figure \ref{fig:examples}.(b)) 
\end{minipage}

\begin{minipage}{9cm}
\normalsize
The results show that Algorithm \ref{alg:ucbd} typically out-performs regular UCB. The exception is the top right experiment
where UCB performs slightly better for $\theta < 0$. This is not surprising, since in this case the structured version of UCB
cannot exploit the additional structure and suffers due to worse constant factors. On the other hand, if $\theta > 0$, then
UCB endures logarithmic regret and performs significantly worse than its structured counterpart. The superiority of Algorithm \ref{alg:ucbd}
would be accentuated in the top left and bottom right experiments by increasing the horizon.
\end{minipage}
\hspace{0.2cm}
\begin{minipage}{4.6cm}
\begin{tikzpicture}[baseline,font=\scriptsize]
  \begin{axis}[ xlabel={$\theta$},
                ylabel={$\hat \E_\theta R_n(\theta)$},
                height=3cm,
                width=4.5cm,
                mark size=0.5pt,
                scaled ticks=false,
                compat=newest]
    \addplot+[only marks,mark=x,mark size=1.5pt] table[x index=0,y index=1] {exp-c-1.txt};
    \addplot+[only marks] table[x index=0,y index=2] {exp-c-1.txt};
  \end{axis}
\end{tikzpicture} \\
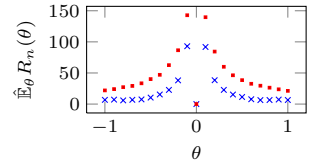
$K = 2$, $\mu_1(\theta) = \theta\ind{\theta > 0}$, \\ $\mu_2(\theta) = -\theta\ind{\theta < 0}$, \\
$n=50\,000$
(see Figure \ref{fig:examples}.(c)) 
\end{minipage}
}

\normalsize

%%%%%%%%%%%%%%%%%%%%%%%%%%%%%%%%%%%%%%%%%%%%%
% CONCLUSION
%%%%%%%%%%%%%%%%%%%%%%%%%%%%%%%%%%%%%%%%%%%%%
\section{Conclusion}\label{sec:conclusions}

%We presented a new algorithm for bandits for which the rewards of the arms are correlated and gave upper and lower bounds on its regret with a special
%focus on the case where finite regret is possible.
%The new results are comparable (but more general than)
%to the results by Bubeck et.\ al.\ \cite{BC13,BPR13} and Agrawal et.\ al. \cite{ATA89}. 

%The take-home is that finite regret is attainable in some situations where both the return of the optimal arm and a
%bound on the minimum gap are unknown. The key criterion for finite regret appears to be that choosing the optimal arm is sufficient to prove that it is indeed optimal. 

The limitation of the new approach is that the proof techniques and algorithm are most suited to the case where the number of actions is relatively small. 
Generalising the techniques to large action spaces is therefore an important open problem.
There is still a small gap between the upper and lower bounds, and the lower bounds have only been proven for special examples. Proving a general problem-dependent
lower bound is an interesting question, but probably extremely challenging given the flexibility of the setting. We are also curious to know if there exist
problems for which the optimal regret is somewhere between finite and logarithmic.
Another question is that of how to define Thompson sampling for structured bandits. Thompson sampling has recently attracted a great deal 
of attention \citep{KKNM12,AG12b,KKM13,AG13,BC13},
but so far we are unable even to define an algorithm resembling Thompson sampling for the general structured bandit problem.
\ifsup
Not only because we have not endowed $\Theta$ with a topology, but also because choosing a reasonable prior seems rather problem-dependent. 
An advantage of our approach is that we do not rely on knowing the distribution of the rewards while with one notable exception \citep{BC13} this
is required for Thompson sampling.
\fi

\ifsup
\else
\vspace{-0.3cm}
\fi
\subsubsect{Acknowledgements} Tor Lattimore was supported by the Google Australia Fellowship for Machine Learning and the Alberta Innovates 
Technology Futures, NSERC.
The majority of this work was completed while R\'emi Munos was visiting Microsoft Research, New England. 
This research was partially supported by the European Community's Seventh Framework Programme under grant agreements no.\ 270327 (project CompLACS).

\bibliographystyle{plain}
\bibliography{d-bandits}

\ifsup
\appendix

\allowdisplaybreaks

%%%%%%%%%%%%%%%%%%%%%%%%%%%%%%%%%%%%%%%%%%%%%
% AMBIGUOUS CASE
%%%%%%%%%%%%%%%%%%%%%%%%%%%%%%%%%%%%%%%%%%%%%
\section{Ambiguous Case}

\newcommand{\thetao}{\theta_{\circ}}
\newcommand{\easy}{\Theta_{\operatorname{easy}}}
\newcommand{\amb}{\Theta_{\operatorname{amb}}}
\newcommand{\hard}{\Theta_{\operatorname{hard}}}

We assume for convenience that $K = 2$ and $\mu_1(\theta) \neq \mu_2(\theta)$ for all $\theta \in \Theta$. The second assumption is non-restrictive, since an
algorithm cannot perform badly on the $\theta$ for which $\mu_1(\theta) = \mu_2(\theta)$, so we can simply remove these points from the parameter space.
Now $\Theta$ can be partitioned into three sets according to whether or not finite regret is
expected by Theorem \ref{thm:finite}, or impossible by Theorem \ref{thm:lower}.
\eq{
\easy &\defined \set{\theta \in \Theta : \exists \epsilon > 0 \text{ such that } \left|\mu_{i^*(\theta)}(\theta') - \mu_{i^*(\theta)}(\theta')\right| < \epsilon \implies
i^*(\theta') = i^*(\theta)} \\
\hard &\defined \set{\theta \in \Theta : \exists \theta' \in \Theta \text{ such that } \mu_{i^*(\theta)}(\theta) = \mu_{i^*(\theta)}(\theta') 
\text{ and } i^*(\theta') \neq i^*(\theta)} \\
\amb &\defined \Theta - \easy - \hard 
}
The topic of this section is to study whether or not finite regret is possible on $\amb$, and what sacrifices need to be made in order to achieve this.
Some examples are given in Figure \ref{fig:ambiguous}. Note that (a) was considered by Bubeck et.\ al.\ \cite[Thm. 8]{BPR13} and will
receive special attention here.
\begin{figure}[H]
\centering
\scriptsize
\begin{tikzpicture}[thick,scale=0.5]
\draw[arm1] (0,2) edge[-] (2,2);
\draw[arm1] (2,2) edge[-] (4,0);
\draw[arm2]  (0,0.1) edge[-] (2,0.1);
\draw[arm2]  (2,2) edge[-] (4,2);
\node at (2,4.5) {(a)};
\node[align=left] at (2,-1.5) {$\amb = [-1,0]$ \\ $\easy = (0,1]$};
\drawgraph
\end{tikzpicture}
\begin{tikzpicture}[thick,scale=0.5]
\draw[arm1] (0,2) edge[-] (2,2);
\draw[arm1] (2,2) edge[-] (4,0);
\draw[arm2]  (0,0) edge[-] (2,2);
\draw[arm2]  (2,2) edge[-] (4,2);
\node at (2,4.5) {(b)};
\node[align=left] at (2,-1.5) {$\amb = [-1,1]$ \\ $\easy = \emptyset$ };
\drawsimplegraph
\end{tikzpicture}
\begin{tikzpicture}[thick,scale=0.5]
\draw[arm1] (0,2) edge[-] (2,2);
\draw[arm1] (2,2) edge[-] (4,3);
\draw[arm2]  (0,0.1) edge[-] (2,0.1);
\draw[arm2]  (2,3.9) edge[-] (4,3.9);
\node at (2,4.5) {(c)};
\node[align=left] at (2,-1.5) {$\amb = [-1,0]$ \\ $\easy = (0,1]$};
\drawsimplegraph
\end{tikzpicture}
\begin{tikzpicture}[thick,scale=0.5]
\draw[arm1] (0,2) edge[-] (2,2);
\draw[arm1] (2,2) edge[-] (4,3);
\draw[arm2]  (0,0.1) edge[-] (2,0.1);
\draw[arm2]  (2,2) edge[-] (4,3.9);
\node[align=left] at (2,-1.5) {$\amb = [-1,0]$ \\ $\easy = (0,1]$};
\node at (2,4.5) {(d)};
\drawsimplegraph
\end{tikzpicture}
\hspace{0.5cm}
\begin{tikzpicture}[thick,scale=0.5]
\draw[thin] (-1,0) rectangle (2.1,4);
\draw (-0.5,2) edge[arm1] node[right=0.3cm] {$\mu_1$} (0.7,2);
\draw (-0.5,1) edge[arm2] node[right=0.3cm] {$\mu_2$} (0.7,1);
\node at (0,3.3) {{\scriptsize \bf Key:}};
\node at (1,-2.1) {\color{white}a hidden message};
\end{tikzpicture}

\caption{Ambiguous examples}\label{fig:ambiguous}
\end{figure}
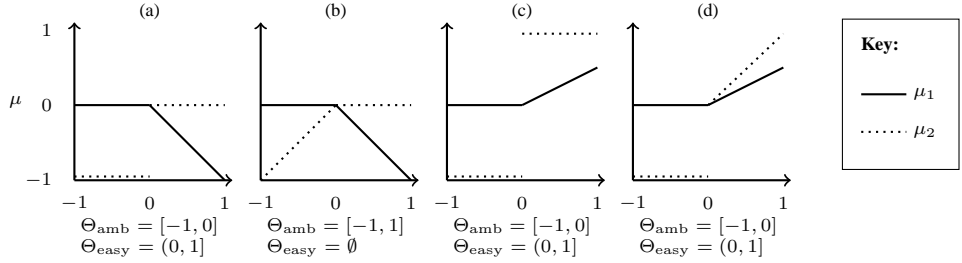

The main theorem is this section shows that finite regret is indeed possible for many $\theta \in \amb$ without incurring significant additional regret
for $\theta \in \easy$ and retaining logarithmic regret for $\theta \in \hard$.
The following algorithm is similar to Algorithm \ref{alg:ucbd}, but favours actions which may be optimal for some plausible ambiguous $\theta$.
Theorems will be given subsequently, but proofs are omitted. \todot{Add proofs?}

\begin{algorithm}[H]
\caption{}
\label{alg:ucbdra}
\begin{algorithmic}[1]
\State {\bf Input:} functions $\mu_1, \cdots, \mu_k:\Theta \to [0,1]$, $\set{\beta_t}_{t=1}^\infty$
\State $\kappa_1 = 0$
\For{$t \in 1,\ldots,\infty$}
\State Define confidence set:
$\displaystyle \tilde\Theta_t \leftarrow \set{\tilde \theta : \forall i,\;\; \left|\mu_i(\tilde \theta) - \hat \mu_{i,T_i(t-1)}\right| < 
\sqrt{{\alpha\sigma^2 \log t \over T_i(t-1)} }}$
\If{$\kappa_t = 0 \;\wedge\; \exists \tilde\Theta_t \cap \amb \neq \emptyset$}
\State Choose $\theta \in \tilde\Theta_t \cap \amb$ arbitrarily and set $\kappa_t = i^*(\theta)$
\EndIf
\State Choose $\displaystyle I_t = \argmax_{k} \sup_{\theta \in \tilde\Theta_t} \mu_k(\theta) + \ind{k = \kappa_t} \sqrt{{\beta_t \log t \over T_k(t-1)}}$
\State $\kappa_{t+1} \leftarrow \ind{I_t = \kappa_t}$
\EndFor
\end{algorithmic}
\end{algorithm}

\begin{theorem}\label{thm:ucbdra}
Suppose $K = 2$, $\theta \in \Theta$ and $i^* = 1$ and $\beta_t = \log \log t$ and $\Delta \defined \mu_1(\theta) - \mu_2(\theta)$.
Then Algorithm \ref{alg:ucbdra} satisfies:
\begin{enumerate}
\item $\limsup_{n\to\infty} \E_\theta R_n(\theta) / \log n < \infty$.
\item If $\theta \in \easy$, then $\E_\theta R_n(\theta) \in O({1 \over \Delta} (\log {1 \over \Delta}) (\log \log {1 \over \Delta}))$.
\item If $\theta$ is such that
\eqn{\label{eq:cond}
\displaystyle \lim_{\delta \to 0} \sup_{\theta' : |\mu_1(\theta) - \mu_1(\theta')| < \delta} {\mu_2(\theta') - \mu_1(\theta') \over |\mu_1(\theta) - \mu_1(\theta')|} < \infty,
}
then 
$\lim_{n\to\infty} \E_\theta R_n(\theta) < \infty$.
\end{enumerate}
\end{theorem}

\begin{remark}
The condition (\ref{eq:cond}) not satisfied for $\theta \in \hard$, since in this case there exists some $\theta'$ with $\mu_1(\theta) = \mu_1(\theta')$
but where $\mu_2(\theta') - \mu_1(\theta') > 0$.
The condition may not be satisfied even for $\theta \in \amb$. See, for example, Figure \ref{fig:ambiguous}.(c). The condition {\it is} satisfied for all
other ambiguous $\theta$ for the problems shown in Figure \ref{fig:ambiguous}.(a,b,d) where the risk $\mu_2(\theta') - \mu_1(\theta')$ decreases linearly
as $\mu_1(\theta) - \mu_1(\theta')$ converges to zero.
\end{remark}

The following theorem shows that you cannot get finite regret for the ambiguous case where (\ref{eq:cond}) is not satisfied without making sacrifices
in the easy case.

\begin{theorem}
Suppose $\theta \in \amb$ with $i^*(\theta) = 1$ 
and $\E_{\theta} R_n(\theta) \in O(1)$. Then there exists a constant $c > 0$ such that for each $\theta'$ with $i^*(\theta') = 2$ we have
\eq{
\E_{\theta'} R_n(\theta') \geq c {\mu_2(\theta') - \mu_1(\theta') \over (\mu_1(\theta) - \mu_1(\theta'))^2}.
}
\end{theorem}

Therefore if the condition (3) in the statement of Theorem \ref{thm:ucbdra} is not satisfied for some ambiguous $\theta$, then we can construct a sequence
$\set{\theta'}_{i=1}^\infty$ such that $\lim_{i\to\infty} \mu_1(\theta'_i) = \mu_1(\theta)$ and where
\eq{
\lim_{i\to\infty} {(\mu_2(\theta'_i) - \mu_1(\theta'))} \E_{\theta'_i} R_n(\theta'_i) = \infty,
}
which means that the regret must grow faster than the inverse of the gap. The situation becomes worse the faster the quantity below diverges to infinity.
\eq{
\sup_{\theta' : |\mu_1(\theta) - \mu_1(\theta')| < \delta} {\mu_2(\theta') - \mu_1(\theta) \over |\mu_1(\theta) - \mu_1(\theta')|}.
}
In summary, finite regret is often possible in the ambiguous case, but may lead to worse regret guarantees in the easy case. Ultimately we are not
sure how to optimise these trade-offs and there are still many interesting unanswered questions.

\subsection*{Analysis of Figure \ref{fig:ambiguous}.(a)}

We now consider a case of special interest that was previously studied by Bubeck et.\ al.\ \cite{BPR13} and is depicted in Figure \ref{fig:ambiguous}.(a).
The structured bandit falls into the ambiguous case when $\theta \leq 0$, since no interval about $\mu_1(\theta) = 0$ is sufficient to rule 
out the possibility that the second action is in fact optimal.
Nevertheless, using a carefully crafted algorithm we show that the optimal regret is smaller than one might expect.
The new algorithm operates in phases, choosing each action a certain number of times. If all evidence points to the first action being best, then 
this is taken until its optimality is proven to be 
implausible, while otherwise the second action is taken. The algorithm is heavily biased towards choosing the first action
where estimation is more challenging, and where the cost of an error tends to be smaller.

\begin{algorithm}[H]
\caption{}
\label{alg:bubeck}
\begin{algorithmic}[1]
\State $\alpha \leftarrow 5$
\For{$\ell \in 2,\ldots,\infty$} \hspace{3cm}\hfill // Iterate over phases
\State $n_{1,\ell} = 2^\ell$ and $n_{2,\ell} = \ell^2$
\State Choose each arm $k \in \set{1,2}$ exactly $n_{k,\ell}$ times and let $\hat \mu_{k,\ell,n_{k,\ell}}$ be the average return
\State $s \leftarrow 0$
\If{$\hat \mu_{1,\ell,n_{1,\ell}} \geq -\sqrt{{\alpha \over n_{1,\ell}} \log \log n_{1,\ell}}$ and $\hat \mu_{2,n_{2,\ell}} < -1/2$}
\While{$\hat \mu_{1,\ell,n_{1,\ell}+s} \geq -\sqrt{{\alpha \log \log (n_{1,\ell}+s) \over n_{1,\ell}+s}}$}
\State Choose action $1$ and $s \leftarrow s + 1$ and $\hat \mu_{1,\ell,n_{1,\ell}+s}$ is average return of arm $1$ this phase
\EndWhile
\Else
\While{$\hat \mu_{2,\ell,n_{2,\ell}+s} \geq -{1 \over 2}$}
\State Choose action $2$ and $s \leftarrow s + 1$ and  $\hat \mu_{2,\ell,n_{2,\ell}+s}$ is average return of arm $2$ this phase
\EndWhile
\EndIf
\EndFor
\end{algorithmic}
\end{algorithm}

\begin{theorem}\label{thm:bubeck}
Let $\Theta = [-1,1]$ and $\mu_1(0) = -\theta\ind{\theta > 0} $ and $\mu_2(\theta) = -\ind{\theta \leq 0}$.
Assume returns are normally distributed with unit variance. 
Then Algorithm \ref{alg:bubeck} 
suffers regret bounded by
\eq{
\E_\theta R_n(\theta) \in \begin{cases} 
O\left({1 \over \theta} \log\log{1 \over \theta}\right) & \text{if } \theta > 0 \\
O(1) & \text{otherwise}.
\end{cases}
}
\end{theorem}

\begin{remark}
Theorem \ref{thm:bubeck} contradicts a result by Bubeck et.\ al.\ \cite[Thm. 8]{BPR13}, which states that for any algorithm 
\eq{
\max\set{\E_0 R_n(0),\; \sup_{\theta > 0} \theta \cdot \E_\theta R_n(\theta)} \in \Omega\left(\log n\right).
}
But by Theorem \ref{thm:bubeck} there exists an algorithm for which 
\eq{
\max\set{\E_0 R_n(0),\;\sup_{\theta > 0} \theta \cdot \E_\theta R_n(\theta)} \in O\left(\sup_{\theta > 0} \min \set{\theta^2 n,\; \log \log{1 \over \theta}} \right)
= O(\log \log n).
}
We are currently unsure whether or not the dependence on $\log\log{1 \over \theta}$ can be dropped from the bound given in Theorem \ref{thm:bubeck}.
Note that Theorem \ref{thm:finite} cannot be applied when $\theta = 0$, so Algorithm \ref{alg:ucbd} suffers logarithmic regret in this case.
Algorithm \ref{alg:bubeck} is carefully tuned and exploits the asymmetry in the problem. 
It is possible that the result of Bubeck et.\ al.\ can be saved in spirit by using
the symmetric structured bandit depicted in Figure \ref{fig:ambiguous}.(b). This would still only give a worst-case bound and does
not imply that finite problem-dependent regret is impossible.
\end{remark}

\begin{proof}[Proof of Theorem \ref{thm:bubeck}]

%\subsubsection*{Step 0: Notation}
It is enough to consider only $\theta \in [0, 1]$, since the returns on the arms is constant for $\theta \in [-1,0]$.
We let $L$ be the number phases (times that the outer loop is executed) and $T_\ell$ be the number of times the sub-optimal
action is taken in the $\ell$th phase. Recall that $\hat \mu_{k,\ell,t}$ denotes the empirical estimate of 
$\mu$ based on $t$ samples taken in the $\ell$th phase.

\subsubsection*{Step 1: Decomposing the regret}
The regret is decomposed:
\eq{
(\theta = 0): &\qquad \E_0 R_n(0) = \E_0 \sum_{\ell=0}^L T_\ell = \sum_{\ell=0}^\infty \Po{0}{L \geq \ell} \E_0[T_\ell|L \geq \ell] \\
(\theta > 0): &\qquad \E_\theta R_n(\theta) = \theta \E_\theta \sum_{\ell=0}^L T_\ell = \theta \sum_{\ell=0}^\infty \Po{\theta}{L \geq \ell} \E_\theta[T_\ell|L \geq \ell] 
}
\subsubsection*{Step 2: Bounding $\E_\theta[T_\ell|L \geq \ell]$}

We need to consider the cases when $\theta = 0$ and $\theta > 0$ separately. If $s \geq 1$, then
\eq{
\Po{0}{T_\ell \geq n_{2,\ell} + s | L \geq \ell} 
&\sr{(a)}\leq \Po{0}{\hat \mu_{2,\ell,n_{2,\ell}+s-1} \geq -{1 \over 2}} 
\sr{(b)}= \Po{0}{\hat \mu_{2,\ell,n_{2,\ell}+s-1} - \mu_2(0) \geq {1 \over 2}} \\
&\sr{(c)}\leq \exp\left(-{1 \over 2} (n_{2,\ell}+s-1)\right) 
\sr{(d)}\leq \exp\left(-{s \over 2}\right),
}
where (a) follows since if the second action is chosen more than $n_{2,\ell}$ times in the $\ell$th phase, then
that phase ends when $\hat \mu_{2,\ell,t} < -{1 \over 2}$, (b) by noting that $\mu_2(0) = -1$,
(c) follows from the standard concentration inequality and the fact that unit variance is assumed,
(d) since $n_{2,\ell} \geq 1$.
Therefore by Lemma \ref{A:lem:expect} we have that $\E_0[T_\ell|L\geq \ell] \leq n_{2,\ell} + 2e^{1/2}$.
Now assume $\theta > 0$ and define 
\eq{
\omega_2(x) = \min\set{z : y \geq x \log \log y,\; \forall y \geq z},
}
which satisfies $\omega_2(x) \in O(x \log \log{x})$.
If $n_{1,\ell} + s - 1 \geq \omega_2\left({4\alpha \over \theta^2}\right)$, then
\eq{
&\Po{\theta}{T_\ell \geq n_{1,\ell} + s | L \geq \ell}
\sr{(a)}\leq \Po{\theta}{\hat \mu_{1,\ell,n_{1,\ell}+s-1} \geq -\sqrt{{\alpha \over n_{\ell,1}+s-1} \log \log (n_{1,\ell} + s - 1)}} \\
&\sr{(b)}= \Po{\theta}{\hat \mu_{1,\ell,n_{1,\ell}+s-1} - \mu_1(\theta) \geq \theta - \sqrt{{\alpha \over n_{\ell,1}+s-1} \log \log (n_{1,\ell}+s-1)}} \\
&\sr{(c)}\leq \Po{\theta}{\hat \mu_{1,\ell,n_{1,\ell}+s-1} - \mu_1(\theta) \geq \theta / 2} 
\sr{(d)}\leq \exp\left(-{\theta^2 \over 8} (n_{1,\ell}+s-1)\right) 
\sr{(e)}\leq \exp\left(-{\theta^2 \over 8} s\right),
}
where (a) follows since if the first arm (which is now sub-optimal) is chosen more than $n_{1,\ell}$ times, then the phase ends if $\hat \mu_{1,\ell,t}$ drops
below the confidence interval. (b) since $\mu_1(\theta) = -\theta$.
(c) since $n_{1,\ell} + s - 1 \geq \omega_2\left({4\alpha \over \theta^2}\right)$.
(d) by the usual concentration inequality and (e) since $n_{1,\ell} \geq 1$.
Another application of Lemma \ref{A:lem:expect} yields 
\eq{
\E_\theta[T_\ell|L \geq \ell] \leq \max\set{n_{1,\ell},\; \omega_2\left({4\alpha \over \theta^2}\right)} + {8e^{\theta^2/8} \over \theta^2}, 
}
where the $\max$ appears because we demanded that $n_{1,\ell}+s-1 \geq \omega_2\left({4\alpha \over \theta^2}\right)$ and since at the
start of each phase the first action is taken at least $n_{1,\ell}$ times before the phase can end.

\subsubsection*{Bounding the number of phases}
Again we consider the cases when $\theta = 0$ and $\theta \geq 0$ separately.
\eqn{
\nonumber\Po{0}{L > \ell} 
&\sr{(a)}\leq \Po{0}{\hat \mu_{2,\ell,n_{2,\ell}} \geq -{1 \over 2} \vee \exists s : \hat \mu_{1,\ell,n_{1,\ell}+s} \leq -\sqrt{{\alpha \over n_{1,\ell}+s} \log \log(n_{1,\ell}+s)}} \\
\nonumber &\sr{(b)}\leq \Po{0}{\hat \mu_{2,\ell,n_{2,\ell}} \geq -{1 \over 2}} + \Po{0}{\exists s : \hat \mu_{1,\ell,n_{1,\ell}+s} \leq -\sqrt{{\alpha \over n_{1,\ell}+s} \log \log(n_{1,\ell}+s)}} \\
\label{eq:ambig-1} &\sr{(c)}\leq \exp\left(-{n_{2,\ell} \over 8}\right) + \Po{0}{\exists s : \hat \mu_{1,\ell,n_{1,\ell}+s} \leq -\sqrt{{\alpha \over n_{1,\ell}+s} \log \log(n_{1,\ell}+s)}}, 
}
where (a) is true since the $\ell$th phase will not end if $\hat \mu_{2,\ell,n_{2,\ell}} < -1/2$ and if $\hat \mu_{1,\ell,t}$ never drops below the confidence
interval. (b) follows from the union bound and (c) by the concentration inequality.
The second term is bounded using the maximal inequality and the peeling technique.
\eq{
&\Po{0}{\exists s : \hat \mu_{1,\ell,n_{1,\ell}+s} \leq -\sqrt{{\alpha \over n_{1,\ell}+s} \log \log(n_{1,\ell}+s)}} \\ 
&\sr{(a)}\leq \sum_{k=0}^\infty \Po{0}{\exists t : 2^k n_{1,\ell} \leq t \leq 2^{k+1} n_{1,\ell} \wedge \hat \mu_{1,\ell,t} \leq -\sqrt{{\alpha \over t} \log \log t}}\\
&\sr{(b)}\leq \sum_{k=0}^\infty \Po{0}{\exists t \leq 2^{k+1}n_{1,\ell} : \hat \mu_{1,\ell,t} \leq -\sqrt{{\alpha \over n_{1,\ell}2^k} \log \log 2^k n_{1,\ell}}}\\
&\sr{(c)}\leq \sum_{k=0}^\infty \exp\left(-\alpha \log \log \left(2^k n_{1,\ell}\right)\right) 
\sr{(d)}= \sum_{k=0}^\infty \left({1 \over \log 2^k + \log n_{1,\ell}}\right)^\alpha 
\sr{(e)}\leq {2 \over \log 2} \left({1 \over \ell \log 2}\right)^{\alpha - 1}
}
where (a) follows by the union bound, (b) by bounding $t$ in the interval $2^k n_{1,\ell} \leq t \leq 2^{k+1} n_{1,\ell}$. (c) follows
from the maximal inequality. (d) is trivial while (e) follows by approximating the sum by an integral.
By combining with (\ref{eq:ambig-1}) we obtain
\eq{
\Po{0}{L > \ell} 
\leq \exp\left(-{n_{2,\ell} \over 8}\right) + {2 \over \log 2} \left({1 \over \ell \log 2}\right)^{\alpha - 1}
= \exp\left(-{\ell \over 8}\right) + {2 \over \log 2} \left({1 \over \ell \log 2}\right)^{\alpha - 1}.
}
More straight-forwardly, if $\theta > 0$, then
\eq{
\Po{\theta}{L > \ell} 
&\leq \Po{\theta}{\exists s : \hat \mu_{2,n_{2,\ell}+s} < -{1 \over 2}} 
\leq 5\exp\left(-{n_{2,\ell} \over 16} \right),
}
where in the last inequality we used Lemma \ref{A:lem:peeling} and naive bounding.

\subsubsection*{Putting it together}
We now combine the results of the previous components to obtain the required bound on the regret. Recall that $\alpha = 5$.
\eq{
(\theta = 0) : \qquad \E_0 R_n(0) 
&= \E_\theta \sum_{\ell=2}^\infty T_\ell 
= \sum_{\ell=2}^\infty \Po{0}{L \geq \ell} \E_0[T_\ell|L\geq\ell] \\
&\leq \sum_{\ell=2}^\infty \left(\exp\left(-{n_{2,\ell} \over 8}\right) + {2 \over \log2} \left({1 \over \ell \log 2}\right)^{\alpha - 1}\right) \left(n_{2,\ell} + 2e^{1/2}\right)\\
&= \sum_{\ell=2}^\infty \left(\exp\left(-{\ell \over 8}\right) + {2 \over \log2} \left({1 \over \ell \log 2}\right)^{\alpha - 1}\right) \left(\ell^2 + 2e^{1/2}\right)
 \in O(1) \\
(\theta > 0) : \qquad \E_\theta R_n(\theta) 
&= \theta \E_\theta \sum_{\ell=2}^\infty T_\ell 
= \theta \sum_{\ell=1}^\infty \P{L \geq \ell} \E[T_\ell|L \geq \ell] \\
&\leq 5\theta \sum_{\ell=2}^\infty  \exp\left(-{n_{2,\ell} \over 16}\right) \left(\max\set{n_{1,\ell},\; \omega_2\left({4\alpha \over \theta^2}\right)} + {8e^{\theta^2/8} \over \theta^2}\right) \\
&\in O\left(\theta \cdot \omega_2\left({\alpha \over \theta^2}\right)\right) 
= O\left({1 \over \theta} \log \log{1 \over \theta}\right).
}
\end{proof}

%%%%%%%%%%%%%%%%%%%%%%%%%%%%%%%%%%%%%%%%%%%%%
% TECHNICAL LEMMAS
%%%%%%%%%%%%%%%%%%%%%%%%%%%%%%%%%%%%%%%%%%%%%
\section{Technical Lemmas}

\begin{lemma}
Define functions $\omega$ and $\omega_2$ by
\eq{
\omega(x) &\defined \min\set{z > 1 : y \geq x \log y,\;\forall y \geq z} \\ 
\omega_2(x) &\defined \min\set{z > e : y \geq x \log \log y,\;\forall y \geq z}.
}
Then
$\omega(x) \in O\left(x \log x\right)$ and $\omega_2(x) \in O\left(x \log \log x\right)$.
\end{lemma}

\begin{lemma}\label{A:lem:peeling}
Let $\set{X_i}_{i=1}^\infty$ be sampled from some sub-gaussian distributed arm with mean $\mu$ and unit sub-gaussian constant. Define $\hat \mu_t = {1 \over t} \sum_{s=1}^t X_s$. Then for $s \geq 6/\Delta^2$ we have
\eq{
\P{\exists t \geq s : \hat \mu_t - \mu \geq \Delta} \leq p + {1 \over \log 2} \log {1 \over 1 - p}
}
where $p = \exp\left(-{s\Delta^2 \over 4}\right)$.
\end{lemma}

\begin{proof}
We assume without loss of generality that $\mu = 0$ and use a peeling argument combined with Azuma's maximal inequality
\eq{
\P{\exists t > s : \hat \mu_t \geq \Delta} 
&\sr{(a)}= \P{\exists k \in \N, 2^k s \leq t < 2^{k+1} s : \hat \mu_t \geq \Delta} \\ 
&\sr{(b)}= \P{\exists k \in \N, 2^k s \leq t < 2^{k+1} s : t \hat \mu_t \geq t\Delta} \\ 
&\sr{(c)}\leq \P{\exists k \in \N, 2^k s \leq t < 2^{k+1} s : t\hat \mu_t \geq 2^{k} s \Delta} \\
&\sr{(d)}\leq \sum_{k=0}^\infty \P{\exists  2^k s \leq t < 2^{k+1} s : t\hat \mu_t \geq 2^{k}s \Delta /2} \\
&\sr{(e)}\leq \sum_{k=0}^\infty \P{\exists t < 2^{k+1} s : t\hat \mu_t \geq 2^{k}s \Delta / 2} \\
&\sr{(f)}\leq \sum_{k=0}^\infty \exp\left(-{1 \over 2} {\left(2^{k} s \Delta \right)^2 \over 2^{k+1} s} \right) 
\sr{(g)}= \sum_{k=0}^\infty \exp\left(-{2^{k} s \Delta^2 \over 4}\right) \\
&\sr{(h)}= \sum_{k=0}^\infty \exp\left(-{s\Delta^2 \over 4}\right)^{2^k} 
\sr{(i)}\leq p + {1 \over \log 2} \log{1 \over 1 - p}
}
where (a) follows by splitting the sum over an exponential grid.
(b) by comparing cumulative differences rather than the means.
(c) since $t > 2^k s$.
(d) by the union bound over all $k$.
(e) follows by increasing the range.
(f) by Azuma's maximal inequality.
(g) and (h) are true by straight-forward arithmetic while (i) follows from Lemma \ref{A:lem:double-exp}.
\end{proof}

\begin{lemma}\label{A:lem:expect}
Suppose $z$ is a positive random variable and for some $\alpha > 0$ it holds for all natural numbers $k$ that
$\P{z \geq k} \leq \exp(-k \alpha)$.
Then $\E z \leq {e^\alpha \over \alpha}$
\end{lemma}

\begin{proof}
Let $\delta \in (0,1)$. Then
\eq{
\P{z \geq \log{1 \over \delta}} 
& \leq \P{z \geq \floor{\log{1 \over \delta}}} 
 \leq \exp\left(-\alpha \floor{\log{1 \over \delta}} \right) \\ 
& \leq e^{\alpha}\cdot \exp\left(-\alpha \log\left({1 \over \delta}\right)\right) 
= e^{\alpha} \cdot {\delta^{\alpha}}.
}
To complete the proof we use a standard identity to bound the expectation
\eq{
\E z 
&\leq \int^1_0 {1 \over \delta} \P{z \geq \log{1 \over \delta}} d\delta 
\leq e^{\alpha} \int^1_0 \delta^{\alpha - 1} d\delta 
=  {e^{\alpha} \over \alpha}.
}
\end{proof}

\begin{lemma}\label{A:lem:double-exp}
Let $p \in (0,7/10)$. Then
$\displaystyle \sum_{k=0}^\infty p^{2^k} \leq p + {1 \over \log 2} \log{1 \over 1 - p}$.
\end{lemma}

\begin{proof}
Splitting the sum and comparing to an integral yields:
\eq{
\sum_{k=0}^\infty p^{2^k} 
&\sr{(a)}= p + \sum_{k=1}^\infty p^{2^k} 
\sr{(b)}\leq p + \int^\infty_0 p^{2^k} dk 
\sr{(c)}= p + {1 \over \log 2} \int^\infty_1 p^u / u du  \\
&\sr{(d)}\leq p + {1 \over \log 2}\sum_{u=1}^\infty p^u / u 
\sr{(e)}=p + {1 \over \log 2} \log{1 \over 1 - p} 
}
where (a) follows by splitting the sum. (b) by noting that $p^{2^k}$ is monotone decreasing and comparing to an integral.
(c) by substituting $u = 2^k$. (d) by reverting back to a sum. (e) follows from a standard formula. 
\end{proof}

%%%%%%%%%%%%%%%%%%%%%%%%%%%%%%%%%%%%%%%%%%%%%
% TABLE OF NOTATION
%%%%%%%%%%%%%%%%%%%%%%%%%%%%%%%%%%%%%%%%%%%%%

\section{Table of Notation}\label{A:notation}

\noindent
\begin{tabular}{p{2cm}p{11cm}}
$K$       & number of arms \\
$\Theta$  & parameter space \\
$\true$  & unknown parameter $\true \in \Theta$ \\ 
$I_t$     & arm played at time-step $t$ \\
$T_i(n)$  & number times arm $i$ has been played after time-step $n$ \\
$X_{i,s}$ & $s$th reward obtained when playing arm $i$ \\
$\gap{i}$ & gap between the means of the best arm and the $i$th arm \\
$\mingap$ & minimum gap, $\mingap \defined \min_{i:\gap{i} > 0} \gap{i}$ \\
$\maxgap$ & maximum gap, $\maxgap \defined \max_{i} \gap{i}$ \\
$A$       & set of arms $A \defined \set{1,2,\cdots, K}$ \\
$A'$      & set of suboptimal arms $A \defined \set{i : \gap{i} > 0}$ \\
$R_n$       & regret at time-step $n$ given unknown true parameter $\true$ \\
$R_n(\theta)$& regret at time-step $n$ given parameter $\theta$ \\
$\mu_i(\theta)$ & mean of arm $i$ given $\theta$ \\
$\hat\mu_{i,s}$& empiric estimate of the mean of arm $i$ after $s$ plays \\
$\mu^*(\theta)$& maximum return at $\theta$. $\mu^*(\theta) \defined \max_i \mu_i(\theta)$ \\
$i^*$       & optimal arm given $\true$ \\
$i^*(\theta)$& optimal arm given $\theta$ \\
$\omega(x)$ & minimum value $y$ such that $z \geq x \log z$ for all $z \geq y$ \\
$\omega_2(x)$ & minimum value $y$ such that $z \geq x \log \log z$ for all $z \geq y$ \\
$F_t$     & event that the true value of some mean is outside the confidence interval about the empiric
estimate at time-step $t$ \\
$\alpha$  & parameter controlling how exploring the algorithm UCB-S is \\
$\sigma^2$& known parameter controlling the tails of the distributions governing the return of the arms \\
$u_i(n)$  & critical number of samples for arm $i$. $u_i(n) \defined \ceil{8\sigma^2\alpha \log n \over \gap{i}^2}$ \\
\end{tabular}

\fi

\end{document}